\ifdefined\XeTeXversion\else\pdfoutput=1\fi 
\documentclass[11pt]{article}

\usepackage[margin=1.1in]{geometry}
\usepackage{amsmath,amssymb,amsthm}
\usepackage{graphicx}
\usepackage{booktabs}
\usepackage{microtype}
\usepackage{natbib}
\usepackage{xcolor}
\usepackage[colorlinks=true,linkcolor=blue!60!black,citecolor=blue!60!black,urlcolor=blue!60!black]{hyperref}

\newtheorem{theorem}{Theorem}[section]
\newtheorem{proposition}[theorem]{Proposition}
\newtheorem{corollary}[theorem]{Corollary}

\theoremstyle{definition}
\newtheorem{assumption}[theorem]{Assumption}
\theoremstyle{remark}
\newtheorem{remark}[theorem]{Remark}

\DeclareMathOperator{\Var}{Var}
\DeclareMathOperator{\Cov}{Cov}
\DeclareMathOperator*{\esssup}{ess\,sup}
\newcommand{\E}{\mathbb{E}}
\newcommand{\Prob}{\mathbb{P}}
\newcommand{\ind}{\mathbf{1}}

\title{Partially Correlated Verifier Cascades in LLM Harnesses:\\
Concave Log-Odds, Polynomial Reliability, and Blind-Spot Ceilings}
\author{Jiangang Han\thanks{Independent researcher. \texttt{jiangangh@gmail.com}.}}
\date{July 2026}

\begin{document}
\maketitle

\begin{abstract}
Serial verification gates are a core reliability primitive in LLM harnesses: a candidate answer is returned only if $k$ verifier calls all accept it. Under conditionally independent gates, the recent \emph{Odds Law} \citep{aksu2026oddslaw} shows that posterior log-odds grow linearly in $k$, so failure decays exponentially; the same work states that ``a tight theory of partially correlated verifier cascades remains open.'' This note gives a minimal such theory. It carries the latent-variable/moment machinery developed for correlated \emph{voting} \citep{liu2026twocalls,liu2026voting} over to the structurally different \emph{conjunctive} verification primitive, where a survivorship effect with no one-shot-voting analogue --- errors that survive $j$ gates are exactly the high-$\alpha$ ones --- is the mechanism behind the concavity, the ceiling, and the trichotomy below. Modeling the per-instance false-accept rate of the verifier \emph{on the generator's own errors} as a latent variable $\alpha\sim G$ (de Finetti), the exact cascade posterior is $\ell_k=\ell_0-\ln m_k$, with $m_k$ the $k$-th moment of $G$, and: (i) $\ell_k$ is \emph{concave} in $k$ for every non-degenerate $G$ --- the Odds Law is its tangent at the first gate and an upper bound; (ii) for $\mathrm{Beta}(a,b)$ latents, failure decays \emph{polynomially}, $1-r_k\asymp k^{-b}$, not exponentially, with all formulas governed by a single correlation parameter $\rho_v=1/(a+b+1)$; (iii) a blind-spot atom of mass $1-\pi$ at $\alpha=1$ caps the total evidence extractable from any number of gates at $-\ln(1-\pi)$ nats, so reliability saturates strictly below $1$; (iv) letting the true-accept rate also vary across instances ($\beta\sim H$) yields a \emph{trichotomy} --- gates eventually always help, plateau, or actively harm --- decided by the upper-tail exponents $b_\alpha$ vs.\ $b_\beta$ of $G$ and $H$, with closed-form crossover $k^\dagger=\frac{a_\alpha b_\beta-a_\beta b_\alpha}{b_\alpha-b_\beta}$; mean gate quality $\bar\Lambda>1$ no longer guarantees that gating helps. Because everything is a functional of $G$, the theory is measurable: $R$ repeated verdicts per instance identify the first $R$ moments of $G$, so \emph{two} verdicts identify $\rho_v$; beta-binomial likelihood and NPMLE recover the full reliability curve and the (tail-dominated, ill-posed) ceiling. Synthetic-recovery experiments validate the estimators and the falsification loop: independence-based extrapolation underestimates the failure rate by $20\times$ at $k=5$ and $\approx3000\times$ at $k=10$ in a realistic regime, while the correlated theory fitted at order $R=8$ tracks held-out depths. The practical lever the theory isolates is \emph{decorrelation} --- changing model family, modality, or evidence source --- rather than adding gates.
\end{abstract}

\section{Introduction}\label{sec:intro}

An LLM harness improves the reliability of an unreliable base model by composing calls: decompose, ensemble, verify, recurse. The lineage of this program is von Neumann's synthesis of reliable organisms from unreliable components \citep{vonneumann1956}, and its most recent and most explicit algebraic form is the \emph{Odds Law} of \citet{aksu2026oddslaw}, with a companion orchestration harness \citep{aksu2026maestro}. For the verification primitive --- pass a candidate answer through $k$ accept/reject gates and return it only if all accept --- the Odds Law is sharp and simple: writing $\ell=\ln\frac{P(\text{correct})}{P(\text{wrong})}$ for log-odds and $\Lambda=\beta/\alpha$ for a gate's likelihood ratio (true-accept over false-accept rate), \emph{conditionally independent} gates each add a fixed evidence increment,
\begin{equation}\label{eq:oddslaw}
\ell_k=\ell_0+k\ln\Lambda ,
\end{equation}
so reliability $r_k=\sigma(\ell_k)$ ($\sigma$ the logistic function) approaches $1$ exponentially fast, and $k=O\!\big(\tfrac{\log(1/\delta)}{\log\Lambda}\big)$ gates suffice for reliability $1-\delta$ \citep[Thm.~5.1]{aksu2026oddslaw}. The same framework contains a threshold dichotomy at $\Lambda^\star=1$ \citep[Thm.~5.2]{aksu2026oddslaw} and a correlation-aware theory of the \emph{voting} primitive via a latent-factor model \citep[Thm.~7.2]{aksu2026oddslaw}.

The assumption doing the work in \eqref{eq:oddslaw} is conditional independence of gate errors. It fails in the situation harnesses actually face: verifiers built from the same model family, prompt style, or training distribution as the generator share \emph{blind spots} with it --- error types the generator likes to produce and the verifier reliably fails to catch. Empirically such blind spots are large: across 14 open models, an average $64.5\%$ of self-generated errors survive self-checking even though the same errors are caught when presented externally \citep{tsui2025selfcorrection}. The broader self-correction literature reaches the same verdict: prompted-LLM self-feedback rarely repairs reasoning errors and can even degrade accuracy, and reliable improvement generally requires \emph{external} feedback rather than more of the same model \citep{huang2024selfcorrect,kamoi2024selfcorrection}. \citet{aksu2026oddslaw} are explicit that this is the open boundary of their algebra: ``a tight theory of \emph{partially} correlated verifier cascades remains open.''

This note supplies a minimal such theory, together with a protocol for \emph{measuring} the correlation it introduces. Methodologically, the latent-variable/moment identification we use is the verification-side counterpart of the two-call framework recently developed for correlated \emph{voting} \citep{liu2026twocalls,liu2026voting}; what is specific to \emph{conjunctive} verification --- the survivorship tilt, and the phenomena it produces (concave depth-scaling, blind-spot ceiling, two-sided trichotomy, a zero-cost internal optimum) --- has no one-shot-voting analogue and is the new content here. Our contributions:

\begin{enumerate}\itemsep2pt
\item \textbf{Exact cascade posterior and concavification} (\S\ref{sec:onesided}). Treating the per-instance false-accept rate as a latent variable $\alpha\sim G$ --- exchangeability across gates plus de Finetti --- gives the exact posterior $\ell_k=\ell_0-\ln m_k$ with $m_k=\E[\alpha^k]$. Since $\ln m_k$ is a cumulant generating function, $\ell_k$ is \emph{concave} in $k$ for every non-degenerate $G$: the Odds Law is the degenerate ($\rho_v\to0$) case, coincides with the true curve only at the first gate, and upper-bounds it everywhere else. The mechanism is a survivorship effect specific to verification: errors that survive $j$ gates are exactly the high-$\alpha$ ones, so late gates face selected survivors.
\item \textbf{Polynomial reliability and a one-parameter family} (\S\ref{sec:onesided}). For $G=\mathrm{Beta}(a,b)$, $1-r_k\sim\kappa\,k^{-b}$: correlation degrades the exponential convergence of \eqref{eq:oddslaw} to polynomial. The single parameter $\rho_v=\Var(\alpha)/(\bar\alpha(1-\bar\alpha))=1/(a+b+1)$ --- the within-instance correlation of two verdicts --- interpolates continuously from the Odds Law ($\rho_v\to0$) to a pure blind-spot model ($\rho_v\to1$). The cost-optimal gate count becomes a power of the value/cost ratio instead of its logarithm.
\item \textbf{Blind-spot ceiling} (\S\ref{sec:onesided}). An atom of mass $1-\pi$ at $\alpha=1$ (errors the verifier never catches) caps the total extractable evidence of the entire cascade at $-\ln(1-\pi)$ nats: $r_\infty=p_0/(p_0+(1-p_0)(1-\pi))<1$ no matter how many gates. This is the verification-side dual of the correlated-voting floor (\citealp[Thm.~7.2]{aksu2026oddslaw}; \citealp{ladha1993}).
\item \textbf{Two-sided trichotomy with closed-form crossover} (\S\ref{sec:twosided}). If the true-accept rate also varies across instances ($\beta\sim H$), then $\ell_k=\ell_0+\ln m_k^{(\beta)}-\ln m_k^{(\alpha)}$ and asymptotically $\ell_k\approx\mathrm{const}+(b_\alpha-b_\beta)\ln k$: gates eventually always help, plateau, or actively harm according to whether the \emph{upper-tail exponent} of $G$ exceeds, equals, or falls below that of $H$. In the harmful regime the reliability peaks at a closed-form crossover $k^\dagger=\frac{a_\alpha b_\beta-a_\beta b_\alpha}{b_\alpha-b_\beta}$ and then decays to zero --- even when the mean likelihood ratio satisfies $\bar\Lambda>1$, so the independence-based dichotomy at $\Lambda^\star=1$ is no longer the right criterion under correlation.
\item \textbf{Identification and inversion of $\rho_v$} (\S\ref{sec:inversion}--\ref{sec:experiments}). All of the above are functionals of $G$, and $G$ is estimable from accept/reject logs alone: with $R$ repeated verdicts per instance on the generator's own errors, $X_i\sim\mathrm{Bin}(R,\alpha_i)$, unbiased U-statistics identify moments up to order $R$ --- so $R=2$ already identifies $\rho_v$. Beta-binomial likelihood recovers the reliability curve; nonparametric MLE recovers atoms. We quantify the intrinsic ill-posedness of the ceiling (an upper-tail functional, boundary resolution $\sim1/R$) and validate the full pipeline in synthetic-recovery experiments, including the falsification loop: fit at low order, predict held-out gate depths.
\end{enumerate}

\begin{figure}[t]
\centering
\includegraphics[width=.72\linewidth]{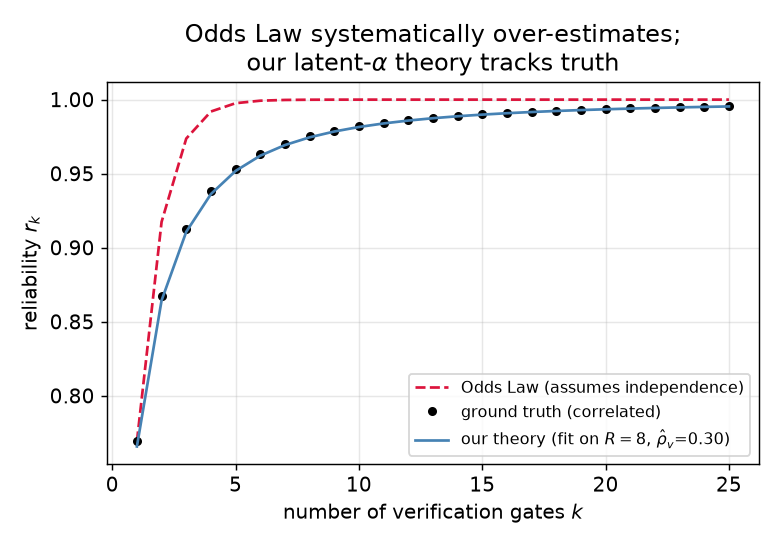}
\caption{The independence-based extrapolation (Odds Law, dashed) versus ground truth in a correlated synthetic world ($\bar\alpha=0.3$, $\rho_v=0.3$, $p_0=0.5$), and our theory fitted \emph{only} on accept counts of order $R=8$ ($\hat\rho_v=0.30$), extrapolated to $k=25$. Exponential extrapolation overestimates reliability; the latent-$\alpha$ theory tracks truth. See \S\ref{sec:experiments}, Experiment D.}
\label{fig:falsification}
\end{figure}

In one sentence, the correction this note makes to the Odds Law: \emph{correlation does not tax the first gate; it taxes the extrapolation.} And the practical lever it isolates: once $\rho_v$ is large, adding gates buys almost nothing --- reliability is bought by \emph{decorrelating} the verifier from the generator (different model family, different modality, external oracles, tool-based checks), which raises $b$, shrinks the blind-spot mass, and lifts the ceiling.

\section{Setup}\label{sec:setup}

\paragraph{Gates and survivor reliability.}
A candidate answer has truth label $C\in\{1,0\}$ with prior $p_0=\Prob(C{=}1)$, prior odds $o_0=p_0/(1-p_0)$, $\ell_0=\ln o_0$. It is passed through $k$ verification gates under \emph{all-accept} (conjunctive) gating: the answer is returned only if all $k$ gates accept. A gate has true-accept rate $\beta=\Prob(\text{accept}\mid C{=}1)$ and false-accept rate $\alpha=\Prob(\text{accept}\mid C{=}0,\text{instance})$. We study the \emph{survivor reliability}
\[
r_k \;=\; \Prob\big(C{=}1 \,\big|\, \text{all }k\text{ gates accept}\big),
\]
the precision of what the cascade returns. In the generate--verify--retry loop that real harnesses run (regenerate until some candidate passes), the returned answer is by construction a survivor, so $r_k$ is the end-to-end correctness of what the harness emits; rejected correct answers cost throughput, not precision. Section~\ref{sec:onesided} sets $\beta\equiv1$ (no false rejections) as an explicit idealization; \S\ref{sec:twosided} removes it.

\paragraph{Latent false-accept rate and the generator--verifier bridge.}
The central modeling step: $\alpha$ is not a constant of the verifier but a property of the \emph{instance}, distributed across instances as
\[
\alpha\sim G \quad\text{supported on }[0,1],
\]
where the instance population is \textbf{the generator's own erroneous outputs}. This bridge matters. $G$'s upper tail --- mass near $\alpha=1$ --- is precisely the set of errors the generator likes to make \emph{and} the verifier fails to catch: the generator--verifier blind-spot alignment. A verifier can be excellent on random errors and still have a heavy $G$-tail on its own generator's errors \citep{tsui2025selfcorrection}; empirically this tail thickens as the generator strengthens --- stronger generators produce errors that are systematically harder to detect \citep{zhou2025variation}. Consequently (\S\ref{sec:inversion}) any measurement of $G$ must use the generator's own errors as the test population, not synthetic or third-party errors.

\begin{assumption}[de Finetti idealization]\label{ass:definetti}
Given the instance (i.e., given $\alpha$), the $k$ gate verdicts are i.i.d.\ $\mathrm{Bernoulli}(\alpha)$ on an erroneous answer (resp.\ $\mathrm{Bernoulli}(\beta)$ on a correct one).
\end{assumption}

This is exact for exchangeable verdicts by de Finetti's theorem, and covers the two operational readings of ``$k$ gates'': $k$ samples of the same verifier at temperature $>0$ (then $\alpha$ is that verifier's per-instance acceptance propensity), or $k$ verifiers from a family sharing blind spots (then $\alpha$ is the family-level propensity). It is a mean-field idealization in the same spirit as von Neumann's constant component-failure probability \citep{vonneumann1956}: all shared structure between gates is compressed into a scalar latent, residual gate-specific correlations are ignored. Everything below is a first-order theory in this sense, and we flag it as the main relaxable assumption.

\section{One-sided theory: concavity, polynomial decay, ceiling}\label{sec:onesided}

Throughout this section $\beta\equiv1$. Write $m_k=\E_{\alpha\sim G}[\alpha^k]$ for the $k$-th moment of $G$, and $\bar\alpha=m_1$.

\begin{proposition}[Exact cascade posterior]\label{prop:exact}
Under Assumption~\ref{ass:definetti},
\begin{equation}\label{eq:exact}
\ell_k=\ell_0-\ln m_k,
\qquad
r_k=\frac{p_0}{p_0+(1-p_0)\,m_k},
\end{equation}
and $r_k$ is nondecreasing in $k$.
\end{proposition}

\begin{proof}
Given $C{=}1$, all gates accept with probability $1$. Given $C{=}0$, the instance carries $\alpha\sim G$ and, conditionally, all $k$ accept with probability $\alpha^k$; marginally $\Prob(\text{all accept}\mid C{=}0)=\E[\alpha^k]=m_k$. Bayes in odds form gives $o_k=o_0/m_k$; take logs and invert. Monotonicity: $m_k$ is nonincreasing since $\alpha\le1$.
\end{proof}

\begin{corollary}[Odds Law as the degenerate case]\label{cor:oddslaw}
If $G=\delta_{\bar\alpha}$ (no correlation), $m_k=\bar\alpha^{\,k}$ and \eqref{eq:exact} reduces to $\ell_k=\ell_0+k\ln\Lambda$ with $\Lambda=1/\bar\alpha$: exactly \eqref{eq:oddslaw}.
\end{corollary}

\begin{theorem}[Concavification]\label{thm:concave}
For any $G$, $k\mapsto\ell_k$ is concave on $k\ge0$; strictly concave unless $G$ is degenerate. Moreover the per-gate evidence increments
\[
\Delta_{j+1}:=\ell_{j+1}-\ell_j=\ln\frac{m_j}{m_{j+1}}=-\ln \E_{(j)}[\alpha],
\qquad
d G_{(j)}\propto \alpha^{j}\,dG,
\]
are strictly decreasing, with $\E_{(j)}[\alpha]\uparrow\esssup\alpha$. Consequently the Odds Law line through the first gate, $\ell_0+k\Delta_1$, upper-bounds $\ell_k$ for all $k\ge1$, with equality only at $k\in\{0,1\}$.
\end{theorem}

\begin{proof}
$\ln m_k=\ln\E[e^{k\ln\alpha}]$ is the cumulant generating function of $\ln\alpha$ evaluated at $k$, hence convex (strictly, unless $\ln\alpha$ is a.s.\ constant); $\ell_k=\ell_0-\ln m_k$ is concave. The increment identity is algebra; $G_{(j)}$ is the $\alpha^j$-tilted (size-biased) distribution, which concentrates on the essential supremum as $j\to\infty$, so $\Delta_{j+1}\downarrow -\ln\esssup\alpha$ (zero when $\esssup\alpha=1$). The tangent bound is concavity.
\end{proof}

\begin{remark}[Survivorship mechanism]\label{rem:survivor}
The tilt $dG_{(j)}\propto\alpha^j dG$ \emph{is} the population of errors still alive after $j$ gates: surviving errors are precisely the ones selected for fooling the verifier. Late gates face survivors, not fresh errors --- which is why their evidence decays to zero. This selection effect is specific to conjunctive verification; it has no analogue in one-shot voting, where all votes face the same instance.
\end{remark}

\begin{theorem}[Exponential $\to$ polynomial]\label{thm:poly}
Suppose $G$ has no atom at $1$ and density $g(\alpha)\sim c\,(1-\alpha)^{b-1}$ as $\alpha\uparrow1$ for some $b,c>0$. Then
\[
m_k\;\sim\; c\,\Gamma(b)\,k^{-b},
\qquad
1-r_k\;\sim\;\kappa\,k^{-b},
\qquad
\kappa=\tfrac{1-p_0}{p_0}\,c\,\Gamma(b).
\]
In particular for $G=\mathrm{Beta}(a,b)$: $m_k=\frac{(a)_k}{(a+b)_k}=\prod_{j=0}^{k-1}\frac{a+j}{a+b+j}$ exactly, and $c\,\Gamma(b)=\frac{\Gamma(a+b)}{\Gamma(a)}$.
\end{theorem}

Proof in Appendix~\ref{app:poly}. The contrast with \eqref{eq:oddslaw} is the headline: under independence $1-r_k\sim\frac{1-p_0}{p_0}\bar\alpha^{\,k}$ decays \emph{exponentially}; any latent heterogeneity with a regularly-varying upper tail degrades this to \emph{polynomial} $k^{-b}$, where $b$ measures how thin the blind-spot tail is. Only the tail exponent matters asymptotically; the Beta family adds exact finite-$k$ formulas.

\begin{theorem}[Blind-spot ceiling]\label{thm:ceiling}
Let $G=\pi\,\mathrm{Beta}(a,b)+(1-\pi)\,\delta_1$ with blind-spot mass $1-\pi\in(0,1)$. Then $m_k\downarrow 1-\pi$ and
\[
\sup_k\,(\ell_k-\ell_0)=-\ln(1-\pi),
\qquad
r_\infty=\frac{p_0}{p_0+(1-p_0)(1-\pi)}<1 .
\]
\end{theorem}

\begin{proof}
$m_k=\pi\,m_k^{\mathrm{Beta}}+(1-\pi)\to1-\pi$ by Theorem~\ref{thm:poly}; plug into \eqref{eq:exact}.
\end{proof}

The entire cascade --- any number of gates from the same correlated family --- carries a finite \emph{evidence budget} of $-\ln(1-\pi)$ nats, set by the blind-spot mass alone, not by $\Lambda$ or $k$. This is the verification-side dual of the correlated-voting floor ($n_{\mathrm{eff}}=1/\gamma$, majority-error floor $\Prob[p(S)<1/2]$) of \citet[Thm.~7.2]{aksu2026oddslaw} and, classically, of correlated-jury theorems \citep{ladha1993}. On the voting side this ceiling is by now also an empirical fact: a panel of nine frontier judges from seven model families supplies only about two independent votes' worth of information, and neither more judges nor smarter aggregation closes the gap \citep{kohli2026ninejudges}.

\begin{proposition}[One correlation parameter]\label{prop:rho}
For two gate verdicts $\ind_1,\ind_2$ on an erroneous instance,
\[
\rho_v:=\mathrm{Corr}(\ind_1,\ind_2)=\frac{\Var(\alpha)}{\bar\alpha(1-\bar\alpha)}
\;\overset{\mathrm{Beta}(a,b)}{=}\;\frac{1}{a+b+1}.
\]
$\rho_v\to0$ (with $\bar\alpha$ fixed) recovers the Odds Law; $\rho_v\to1$ recovers a two-point blind-spot model ($\alpha\in\{0,1\}$), where gates either succeed immediately or never.
\end{proposition}

\begin{proof}
$\E[\ind_1\ind_2]=\E[\alpha^2]=m_2$, $\E[\ind_i]=\bar\alpha$, so $\Cov=m_2-\bar\alpha^2=\Var(\alpha)$ and $\Var(\ind_i)=\bar\alpha(1-\bar\alpha)$. For Beta, $\Var(\alpha)=\frac{ab}{(a+b)^2(a+b+1)}$.
\end{proof}

$\rho_v$ is the same intraclass-correlation functional that governs the voting side \citep{ladha1993,aksu2026oddslaw}, now appearing on the verification side, and --- unlike a modeling parameter --- it is directly measurable (\S\ref{sec:inversion}).

\begin{corollary}[Cost-optimal gate count]\label{cor:kstar}
With per-gate cost $c$, success value $U$, and objective $J(k)=U r_k-ck$, the optimum under Theorem~\ref{thm:poly} scales as
\[
k^{*}\approx\Big(\frac{U\kappa b}{c}\Big)^{\!1/(b+1)}
\quad\text{(power law)},
\qquad\text{vs.}\qquad
k^{*}_{\mathrm{indep}}\sim\frac{\ln(U/c)}{\ln(1/\bar\alpha)}
\quad\text{(logarithmic)}.
\]
\end{corollary}

Proof in Appendix~\ref{app:kstar}. Correlation is a double penalty: each gate buys less, and reaching a target reliability requires polynomially rather than logarithmically many gates --- until the ceiling makes the target unreachable altogether.

\paragraph{How large is the error of assuming independence?}
Table~\ref{tab:overestimate} evaluates \eqref{eq:exact} in a moderate regime: $p_0=0.5$, $\bar\alpha=0.3$ (a decent verifier: catches $70\%$ of errors per gate), $\rho_v=0.3$ ($a=0.7$, $b\approx1.63$), against the Odds Law with the same $\bar\alpha$. The curves agree at $k=1$ by construction and then split: by $k=5$ the Odds Law claims near-perfection ($99.8\%$) while the truth is $95.3\%$ --- a $20\times$ underestimate of the failure rate; by $k=10$, $3000\times$. An operator budgeting gates by the independence formula believes they bought five nines; they bought $98\%$.

\begin{table}[t]
\centering
\caption{Independence-based extrapolation vs.\ correlated truth ($p_0=0.5$, $\bar\alpha=0.3$, $\rho_v=0.3$).}
\label{tab:overestimate}
\begin{tabular}{r cc cc c}
\toprule
$k$ & $r_k$ (indep.) & $r_k$ (true) & $1-r_k$ (indep.) & $1-r_k$ (true) & failure underest.\\
\midrule
1  & 0.769 & 0.769 & 0.231 & 0.231 & $1\times$\\
2  & 0.917 & 0.867 & 0.083 & 0.133 & $1.6\times$\\
3  & 0.974 & 0.913 & 0.026 & 0.087 & $3.3\times$\\
5  & 0.998 & 0.953 & $2.4\!\times\!10^{-3}$ & 0.047 & $20\times$\\
10 & 0.999994 & 0.982 & $5.9\!\times\!10^{-6}$ & 0.018 & $\approx3000\times$\\
\bottomrule
\end{tabular}
\end{table}

\section{Two-sided theory: when gates help, plateau, or harm}\label{sec:twosided}

Real verifiers also \emph{falsely reject}: some correct answers --- valid but unidiomatic code, unusual phrasings --- are systematically refused. Let the per-instance true-accept rate be latent too, $\beta\sim H$ on the population of the generator's \emph{correct} outputs, with moments $m_k^{(\beta)}=\E[\beta^k]$ (and, for symmetry, write $m_k^{(\alpha)}\equiv m_k$ for the $\alpha$-side moments of \S\ref{sec:onesided}); keep Assumption~\ref{ass:definetti} on both sides. The same Bayes computation gives
\begin{equation}\label{eq:twosided}
\ell_k=\ell_0+\ln m_k^{(\beta)}-\ln m_k^{(\alpha)} :
\end{equation}
a race between the (concave, saturating) benefit of filtering errors and the accumulating cost of killing correct answers. Section~\ref{sec:onesided} is the special case $H=\delta_1$.

\begin{theorem}[Trichotomy]\label{thm:trichotomy}
Let $G$ and $H$ have no atoms at $1$ and regularly-varying upper tails with exponents $b_\alpha$ and $b_\beta$ respectively (densities $\sim c_\alpha(1-x)^{b_\alpha-1}$, $c_\beta(1-x)^{b_\beta-1}$ at $x\uparrow1$). Then
\[
\ell_k=\ell_0+\ln\frac{c_\beta\,\Gamma(b_\beta)}{c_\alpha\,\Gamma(b_\alpha)}+(b_\alpha-b_\beta)\ln k+o(1),
\qquad k\to\infty,
\]
so exactly one of three regimes obtains:
\begin{center}
\begin{tabular}{lll}
\emph{(i) gates eventually always help:} & $b_\alpha>b_\beta$ & $\ell_k\to+\infty$, \ $1-r_k\asymp k^{-(b_\alpha-b_\beta)}$;\\
\emph{(ii) plateau:} & $b_\alpha=b_\beta$ & $r_k\to\sigma\big(\ell_0+\ln\tfrac{c_\beta}{c_\alpha}\big)<1$;\\
\emph{(iii) gates eventually harm:} & $b_\alpha<b_\beta$ & $\ell_k\to-\infty$, \ $r_k\to0$.
\end{tabular}
\end{center}
\end{theorem}

Proof in Appendix~\ref{app:trichotomy}. The criterion is a tail comparison, with a plain-language reading: \emph{whichever side runs out of near-unanimous instances first, loses.} $b_\alpha$ large means errors that ``almost always fool the verifier'' are rare (good); $b_\beta$ large means correct answers that ``almost always pass'' are rare --- the verifier keeps finding reasons to reject good answers --- and then deep cascades kill the correct population faster than the erroneous one.

\begin{corollary}[The independence threshold is not the right criterion under correlation]\label{cor:threshold}
Under conditional independence, gating helps iff the mean likelihood ratio exceeds one ($\Lambda^\star=1$ dichotomy, \citealp[Thm.~5.2]{aksu2026oddslaw}). Under correlation, $\bar\Lambda=\bar\beta/\bar\alpha>1$ guarantees only that the \emph{first} gate helps ($\delta_1=\ln\bar\Lambda>0$); the eventual direction is decided by the tail exponents $b_\alpha$ vs.\ $b_\beta$, which are logically independent of $\bar\Lambda$. Table~\ref{tab:twosided} exhibits $\bar\Lambda=2.2$ with $r_k\to0$.
\end{corollary}

\begin{proposition}[Closed-form crossover for Beta tails]\label{prop:kdagger}
For $\alpha\sim\mathrm{Beta}(a_\alpha,b_\alpha)$, $\beta\sim\mathrm{Beta}(a_\beta,b_\beta)$, the net evidence of gate $j{+}1$ is
\[
\delta_{j+1}=\ln\frac{\E^{H}_{(j)}[\beta]}{\E^{G}_{(j)}[\alpha]}
=\ln\frac{(a_\beta+j)/(a_\beta+b_\beta+j)}{(a_\alpha+j)/(a_\alpha+b_\alpha+j)} ,
\]
i.e., gate $j{+}1$ helps iff, \emph{among survivors of the first $j$ gates}, correct answers are still accepted more often than surviving errors. The sign of $\delta_{j+1}$ changes at most once in $j$, at
\[
k^\dagger=\frac{a_\alpha b_\beta-a_\beta b_\alpha}{b_\alpha-b_\beta} ,
\]
so in regime (iii) with $\delta_1>0$ the reliability is unimodal in $k$ with discrete optimum $k_{\mathrm{opt}}=\lceil k^\dagger\rceil$.
\end{proposition}

Proof in Appendix~\ref{app:kdagger}. Note $k^\dagger$ exists at \emph{zero} gate cost: this internal optimum is driven purely by the two selection effects, and is distinct from the cost-driven $k^*$ of Corollary~\ref{cor:kstar}.

\paragraph{Numerical example.}
Take a verifier that is decent on both sides \emph{on average}: $\bar\alpha=0.3$ ($\alpha\sim\mathrm{Beta}(0.7,1.63)$, as in Table~\ref{tab:overestimate}) and $\bar\beta=0.67$ ($\beta\sim\mathrm{Beta}(8,4)$). Mean gate quality is healthy: $\bar\Lambda=2.2$, first-gate evidence $\delta_1=0.80>0$. But $b_\alpha=1.63<b_\beta=4$: near-unanimously-accepted correct answers are scarcer than near-undetectable errors, so this is regime (iii), with $k^\dagger=\frac{0.7\cdot4-8\cdot1.63}{1.63-4}\approx4.3$. Table~\ref{tab:twosided}: reliability peaks at $k=5$ ($78.7\%$), then declines --- back to $62\%$ by $k=20$ and heading to zero as $\ell_k\sim-2.37\ln k$ --- while the independence extrapolation reports five nines and rising. Neither the Odds Law nor the one-sided model can represent this reversal; it is a joint effect of the two selection pressures. The reversal is not merely theoretical: without external feedback, LLM self-correction can \emph{lower} reasoning accuracy rather than raise it \citep{huang2024selfcorrect} --- the empirical signature of a same-family gate that harms.

\begin{table}[t]
\centering
\caption{Two-sided cascade: unimodal truth vs.\ monotone independence prediction ($p_0=0.5$, $\bar\alpha=0.3$, $\bar\beta=0.67$; $\alpha\sim\mathrm{Beta}(0.7,1.63)$, $\beta\sim\mathrm{Beta}(8,4)$; $k^\dagger\approx4.3$).}
\label{tab:twosided}
\begin{tabular}{r cc l}
\toprule
$k$ & $r_k$ (indep.) & $r_k$ (true) & \\
\midrule
1  & 0.690 & 0.690 & \\
2  & 0.832 & 0.751 & \\
3  & 0.917 & 0.776 & \\
5  & 0.982 & \textbf{0.787} & peak ($\approx k^\dagger$)\\
8  & 0.998 & 0.771 & declining\\
10 & 0.99966 & 0.751 & \\
15 & 0.999994 & 0.691 & \\
20 & $\approx1$ & 0.623 & still declining\\
\bottomrule
\end{tabular}
\end{table}

\begin{remark}[A spectrum of earlier models]
$\beta\equiv1$ is $b_\beta=0$: regime (i), gates monotonically help (\S\ref{sec:onesided}). A \emph{constant} $\beta_0<1$ is an infinitely thin tail ($m_k^{(\beta)}=\beta_0^k$, ``$b_\beta=\infty$''): always regime (iii), with the harsher linear decay $\ell_k\sim k\ln\beta_0$. The two-sided Beta model interpolates between these extremes and shows the boundary is a tail comparison, not a side condition. With atoms on both sides ($1-\pi_\alpha$ at $\alpha{=}1$, $1-\pi_\beta$ at $\beta{=}1$) the ceiling generalizes to $r_\infty=\frac{p_0(1-\pi_\beta)}{p_0(1-\pi_\beta)+(1-p_0)(1-\pi_\alpha)}$: what survives at depth is the ratio of the two blind-spot masses.
\end{remark}

\section{Measuring $\rho_v$: an inversion protocol}\label{sec:inversion}

Everything above is a functional of $G$ (and $H$); none of it is hypothetical, because $G$ is estimable from accept/reject logs alone. Prior correlation-aware analyses \citep{aksu2026oddslaw,aksu2026maestro} treat the correlation as given and validate by Monte Carlo; to our knowledge no one has \emph{measured} a verifier-cascade correlation on a real generator--verifier pair. The forward model is standard empirical Bayes \citep{robbins1956}:

\begin{enumerate}\itemsep1pt
\item On a calibration set with ground truth, collect \textbf{the generator's own erroneous outputs} (\S\ref{sec:setup}; using third-party errors measures the wrong $G$).
\item For each erroneous instance $i$, sample the verifier $R$ times (temperature $>0$); record accept counts
\[
X_i\mid\alpha_i\sim\mathrm{Bin}(R,\alpha_i),\qquad \alpha_i\sim G .
\]
\item Recover $G$ (binomial deconvolution), or directly its low-order functionals.
\end{enumerate}

\begin{proposition}[Moment identification; two verdicts suffice for $\rho_v$]\label{prop:ident}
For $k\le R$, $\;\widehat{m_k}=\frac{1}{N}\sum_i\binom{X_i}{k}\big/\binom{R}{k}$ is unbiased for $m_k$; hence $\{X_i\}$ identifies $m_1,\dots,m_R$, and
\[
\widehat{\rho_v}=\frac{\widehat{m_2}-\widehat{m_1}^2}{\widehat{m_1}(1-\widehat{m_1})}
\]
is consistent already at $R=2$.
\end{proposition}

\begin{proof}
$\E\big[\binom{X}{k}\mid\alpha\big]=\binom{R}{k}\alpha^k$ (binomial factorial moments); integrate over $G$ and apply Proposition~\ref{prop:rho}.
\end{proof}

Two further estimators of increasing resolution: \textbf{(M2)} beta-binomial maximum likelihood for $(\hat a,\hat b)$, giving the full predicted curve $r_k$ and ceiling; \textbf{(M3)} nonparametric MLE over mixing distributions \citep{kieferwolfowitz1956,efron2016}, which does not assume Beta and is the only one able to expose an atom at $\alpha=1$ (true blind spots) or multimodality.

\paragraph{Ill-posedness of the ceiling.}
The inversion is a textbook ill-posed inverse problem, and honesty about resolution is part of the protocol:
\textbf{(P1)} the ceiling is an upper-tail functional, and $R$ verdicts cannot distinguish $\alpha=1$ from $\alpha=1-\epsilon$ below boundary resolution $\epsilon\sim1/R$: two worlds with ceilings $0.91$ and $1.00$ produce nearly identical data at small $R$ (Fig.~\ref{fig:illposed}), so ceiling estimates carry an $R$-dependent identifiability floor and require regularization --- the exact structure (resolution kernels, damped inversion) long formalized for gross Earth data \citep{backusgilbert1968};
\textbf{(P2)} $R$ verdicts identify only the first $R$ moments: cheap protocols pin $\rho_v$ but not the deep-$k$ behavior;
\textbf{(P3)} with a labeling budget $B=NR$ there is an accuracy trade between instances and depth; low-order functionals favor large $N$, tail functionals demand large $R$.

\paragraph{Falsification loop.}
The theory earns its keep by out-of-sample prediction: fit $G$ at low order (small $R$), \emph{extrapolate} the entire curve $r_k$ to held-out gate depths, and compare. Exponential ($\rho_v=0$) and polynomial ($\rho_v>0$) predictions separate fast (Table~\ref{tab:overestimate}), so modest data decide. A practical decision rule falls out: measure $\widehat{\rho_v}$ with $R=2$; if small, gates are cheap reliability (independence regime); if large, stop buying gates and spend on decorrelation --- a different model family or modality, external oracles, tool-based verification. Even trivial perturbations that break the shared-blind-spot channel are known to help disproportionately \citep{tsui2025selfcorrection}.

\paragraph{Decorrelation vs.\ the exchangeability assumption.}
One tension deserves to be stated head-on rather than left to the caveats. The lever the theory recommends --- decorrelate the verifier from the generator by changing model family, modality, or evidence source --- deliberately makes the gates \emph{heterogeneous}, whereas Assumption~\ref{ass:definetti} treats them as exchangeable under a single scalar $\alpha$. The two are reconciled by reading ``$k$ gates'' at the right granularity. \textbf{(A)}~When the $k$ gates are repeated draws of one verifier, or members of one blind-spot-sharing family, the scalar model is exact and the message is the pessimistic one: the extra gates inherit the same tail, so reliability saturates at the ceiling. \textbf{(B)}~When the gates come from genuinely different families they are no longer exchangeable; the faithful object is a \emph{vector} latent $\boldsymbol{\alpha}=(\alpha_1,\dots,\alpha_k)$ (or a hierarchical $G$), and the present scalar theory is its first-order projection. That projection already points the right way: within the scalar model, splicing in a less-correlated gate is exactly what \emph{thins} the effective upper tail of $G$ --- raising the exponent $b$, shrinking the blind-spot mass $1-\pi$, and lifting the ceiling $-\ln(1-\pi)$. Decorrelation is thus not outside the theory's logic but the operation that moves $G$ in the one direction the scalar theory says matters; the faithful heterogeneous treatment --- e.g.\ a rank-one shared-plus-family latent $\alpha_i=\sigma(u+v_i)$, under which \emph{partial} decorrelation registers as a measurable rise in $b$ --- is the natural sequel (\S\ref{sec:limits}).

\section{Synthetic-recovery validation}\label{sec:experiments}

Before touching real logs, we validate that the pipeline recovers known ground truth --- the synthetic-recovery discipline standard in geophysical inversion. All experiments: $N=4000$ instances, fixed seeds; code to reproduce all tables and figures is available at \url{https://github.com/jianganghan/harness-verifier-cascades}.

\begin{description}\itemsep2pt
\item[A (full-spectrum recovery).] Data generated with $\rho_v\in\{0.05,\dots,0.50\}$, $\bar\alpha=0.3$, $R=10$: the moment estimator recovers $0.050/0.291/0.502$ at true $0.05/0.30/0.50$; M2 agrees.
\item[B (two verdicts suffice).] True $\rho_v=0.30$: $R=2$ yields $\widehat{\rho_v}=0.274$; $R\ge3$ is essentially exact --- Proposition~\ref{prop:ident} in action.
\item[C (ceiling ill-posedness).] Two worlds with $10\%$ atom mass at $\alpha=1.00$ vs.\ at $\alpha=0.97$ (true ceilings $r_\infty=0.91$ vs.\ $1.00$): at $R=5$ their accept-count histograms are nearly indistinguishable (Fig.~\ref{fig:illposed}, left); at $R=50$ the atom emerges and NPMLE assigns it mass $0.100$ vs.\ $0.000$ (right). $\rho_v$ is cheap; the ceiling is expensive --- exactly (P1)/(P2).
\item[D (falsification loop).] Fit beta-binomial on accept counts of order $R=8$ only ($\widehat{\rho_v}=0.30$), extrapolate $r_k$ to $k\le25$: at $k=5$ the correlated theory predicts $0.954$ against truth $0.953$, while the independence extrapolation from the same first-gate data gives $0.998$ (Fig.~\ref{fig:falsification}).
\end{description}

\begin{figure}[t]
\centering
\includegraphics[width=.95\linewidth]{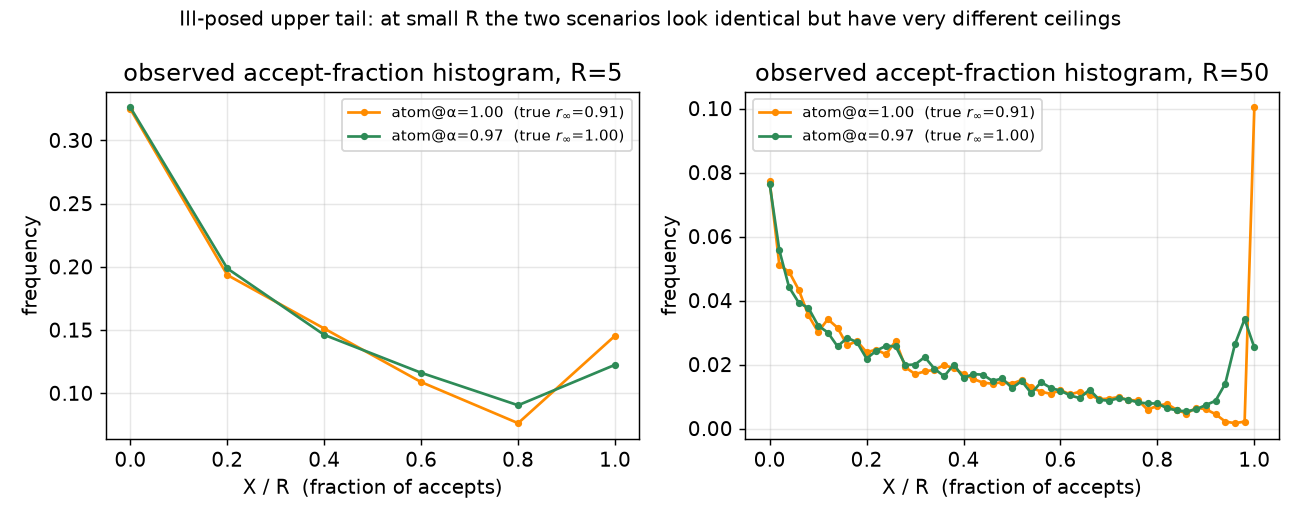}
\caption{Ill-posed upper tail (Experiment C). Two worlds --- $10\%$ blind-spot atom at $\alpha=1.00$ (ceiling $0.91$) vs.\ atom at $\alpha=0.97$ (ceiling $1.00$) --- are observationally near-identical at $R=5$ (left) and separate only at $R=50$ (right), where the $\alpha=1$ spike appears at $X/R=1$. The up-turn near $X/R{=}1$ is not noise but the signature of this ceiling atom: the $10\%$ of near-perfect items accept on (almost) every one of the $R$ gates, piling up at the right edge --- deterministically at $X{=}R$ for $\alpha{=}1$ (sharp orange spike), and as a $\mathrm{Binomial}(R,0.97)$ cluster just below the edge for $\alpha{=}0.97$ (broader green shoulder). The ceiling is a tail property with an $R$-dependent identifiability floor.}
\label{fig:illposed}
\end{figure}

The immediate next step is the same estimators on \emph{real} accept/reject logs: tasks with programmatic ground truth (unit-tested code, exact-match extraction, numerically checkable math), the generator's own wrong answers as the instance population, $R$ verifier samples per instance --- the protocol and code require only the data source to change.

\section{Related work}\label{sec:related}

\paragraph{Reliability algebra for harnesses.}
The direct target is the Odds Law / Maestro Order pair \citep{aksu2026oddslaw,aksu2026maestro}: four primitives with composition laws, the $\Lambda^\star=1$ threshold dichotomy, a water-filling controller, and --- on the \emph{voting} side --- a latent-factor correlation theory with an effective-sample floor (Thm.~7.2 there). We modify exactly one assumption (conditional independence of verification gates) and answer the open problem stated there; our Corollary~\ref{cor:oddslaw} recovers their Lemma~4.3/Thm.~5.1 as the $\rho_v\to0$ boundary. The von Neumann lineage \citep{vonneumann1956} is shared.

\paragraph{Correlated voting and ensembling.}
Exchangeable-vote Condorcet theory dates to \citet{ladha1993}; recent LLM-side treatments include non-monotone vote-scaling from difficulty heterogeneity \citep{chen2024morecalls}, general de Finetti characterizations of when voting helps or hurts \citep{liu2026voting}, and two-call moment identification of vote correlation \citep{liu2026twocalls} --- the voting-side analogue of our Proposition~\ref{prop:ident}. At the system level, hierarchical majority trees with a shared-error correlation exhibit a Kesten--Stigum-style amplification--collapse phase transition \citep{liu2026synergy}, and for answer-\emph{selection} policies over a model pool (routing, voting, model cascades), \citet{chen2026cofailure} prove an accuracy ceiling of $1-\beta$ --- the pool's simultaneous-failure rate --- together with a negative identification result: pairwise error correlation cannot determine $\beta$. Both concern correlated \emph{generators}; our ceiling concerns correlated \emph{checking} --- the blind-spot atom of a generator--verifier family under conjunctive gating. The non-identifiability of \citet{chen2026cofailure} does not apply to Proposition~\ref{prop:ident}, whose observable is repeated verdicts on the \emph{same} instance rather than cross-model pairwise statistics; it is, however, consonant with our diagnosis that the tail quantity setting the ceiling is exactly the ill-posed direction of the inversion (\S\ref{sec:inversion}). Verification is not voting: conjunctive gating induces the survivorship tilt of Remark~\ref{rem:survivor}, which has no one-shot-voting analogue, and yields phenomena absent there (per-gate evidence decay, the two-sided trichotomy, an internal $k^\dagger$ at zero cost).

\paragraph{Self-verification blind spots.}
That models systematically miss their own errors is documented empirically \citep{tsui2025selfcorrection} and argued information-theoretically via a shared latent failure cause capping what self-evaluation can add \citep{limits2026}. We contribute the gating-algebra form of this idea: explicit finite-$k$ posteriors, the concavity/polynomial/ceiling structure, closed forms in a one-parameter family --- and a measurement protocol, which the information-theoretic treatment explicitly lacks.

\paragraph{Verifier/judge scaling and the exponential--polynomial dichotomy.}
\citet{halder2026judge} obtain a finite optimal amount of judge-guided sampling when the reward model is misspecified; our regime (iii) is a distinct mechanism (correlation, not bias) for the same phenomenology, and the two are separable in data because $\rho_v$ is directly measurable. That test-time reliability can decay \emph{polynomially rather than exponentially} is by now a recurring finding: a knockout tournament's failure probability decays exponentially \emph{or} as a power law depending on how one scales \citep{chen2024provable}, best-of-$k$ under a misspecified reward yields polynomially diminishing returns \citep{halder2026judge}, and attack-success curves show an explicit polynomial--exponential crossover set by the underlying generative mechanism \citep{halder2026jailbreak}. We therefore do not claim the dichotomy itself as new; our contribution is to fix it to a specific mechanism --- conjunctive survivorship against a blind-spot tail --- for which the failure exponent is the closed form $k^{-b}$ in $G$'s upper-tail index $b$, controlled by the single measurable parameter $\rho_v$. Empirical verification studies independently report that error-detectability falls with generator strength and that verifier scaling alone cannot overcome these correlation-driven limits \citep{zhou2025variation}.

\section{Limitations and outlook}\label{sec:limits}

The theory is deliberately minimal. (1) Assumption~\ref{ass:definetti} compresses all inter-gate structure into a scalar exchangeable latent --- a mean-field idealization; gate-specific systematic differences (heterogeneous verifier families) call for a vector latent. (2) All-accept semantics: under generate--verify--retry, false rejections cost throughput rather than precision, which restores the one-sided picture for precision while making \S\ref{sec:twosided} the right model when regeneration is impossible or gates are terminal. (3) Beta is a convenience for closed forms; the asymptotics need only regularly-varying tails, and M3 removes the parametric assumption in estimation. (4) The validation here is synthetic recovery; the measurement on real generator--verifier pairs --- the quantitative bridge both \citet{aksu2026oddslaw} and we regard as the natural sequel --- is in progress, and the protocol of \S\ref{sec:inversion} is designed so that only the data source changes.

\appendix

\section{Deferred proofs}

\subsection{Theorem~\ref{thm:poly} (polynomial decay)}\label{app:poly}
\begin{proof}
Substitute $\alpha=1-s$: $m_k=\int_0^1(1-s)^k g(1-s)\,ds$. The integrand is dominated by $s=O(1/k)$; with $g(1-s)\sim c\,s^{b-1}$ and $(1-s)^k=e^{k\ln(1-s)}\sim e^{-ks}$ on that scale, Watson's lemma gives $m_k\sim c\int_0^\infty e^{-ks}s^{b-1}ds=c\,\Gamma(b)\,k^{-b}$. For $G=\mathrm{Beta}(a,b)$, $m_k=\frac{B(a+k,b)}{B(a,b)}=\frac{\Gamma(a+b)}{\Gamma(a)}\cdot\frac{\Gamma(a+k)}{\Gamma(a+b+k)}\sim\frac{\Gamma(a+b)}{\Gamma(a)}k^{-b}$ by Stirling. Finally $1-r_k=\frac{(1-p_0)m_k}{p_0+(1-p_0)m_k}\sim\frac{1-p_0}{p_0}m_k$.
\end{proof}

\subsection{Corollary~\ref{cor:kstar} (cost-optimal gates)}\label{app:kstar}
\begin{proof}
With $1-r_k\approx\kappa k^{-b}$, marginal value $U\frac{d r_k}{dk}\approx U\kappa b\,k^{-(b+1)}$; setting it equal to $c$ gives $k^{*}=(U\kappa b/c)^{1/(b+1)}$. Under independence $1-r_k\approx\frac{1-p_0}{p_0}\bar\alpha^k$, and equating $U$ times its derivative to $c$ gives $k^{*}=\Theta(\ln(U/c)/\ln(1/\bar\alpha))$.
\end{proof}

\subsection{Theorem~\ref{thm:trichotomy} (trichotomy)}\label{app:trichotomy}
\begin{proof}
Apply Appendix~\ref{app:poly} to both moment sequences: $m_k^{(\alpha)}\sim c_\alpha\Gamma(b_\alpha)k^{-b_\alpha}$, $m_k^{(\beta)}\sim c_\beta\Gamma(b_\beta)k^{-b_\beta}$. Then \eqref{eq:twosided} gives $\ell_k=\ell_0+\ln\frac{c_\beta\Gamma(b_\beta)}{c_\alpha\Gamma(b_\alpha)}+(b_\alpha-b_\beta)\ln k+o(1)$. The three regimes read off the sign of $b_\alpha-b_\beta$; in regime (i), $1-r_k\sim\frac{1-p_0}{p_0}\frac{m_k^{(\alpha)}}{m_k^{(\beta)}}\asymp k^{-(b_\alpha-b_\beta)}$.
\end{proof}

\subsection{Proposition~\ref{prop:kdagger} (crossover)}\label{app:kdagger}
\begin{proof}
For Beta, the $x^j$-tilted distribution of $\mathrm{Beta}(a,b)$ is $\mathrm{Beta}(a+j,b)$ with mean $\frac{a+j}{a+b+j}$, and $\delta_{j+1}=\ln\frac{m^{(\beta)}_{j+1}/m^{(\beta)}_j}{m^{(\alpha)}_{j+1}/m^{(\alpha)}_j}$ is the stated ratio of tilted means. Setting $\frac{a_\beta+j}{a_\beta+b_\beta+j}=\frac{a_\alpha+j}{a_\alpha+b_\alpha+j}$ and cross-multiplying, the $j^2$ terms cancel, leaving the linear equation $j(b_\alpha-b_\beta)=a_\alpha b_\beta-a_\beta b_\alpha$, whence $k^\dagger$; linearity implies at most one sign change. If $b_\alpha<b_\beta$ and $\delta_1>0$, increments are positive for $j<k^\dagger$ and negative after, so $\ell_k$ (hence $r_k$) is unimodal with integer maximizer $\lceil k^\dagger\rceil$.
\end{proof}

\bibliographystyle{plainnat}
\bibliography{references}

@misc{aksu2026oddslaw,
  author = {Hidayet Aksu},
  title  = {Odds Law: The Decomposition Algebra On How Intelligence Organizes Itself to Solve Difficult Problems Reliably},
  year   = {2026},
  note   = {arXiv:2606.15712}
}

@misc{aksu2026maestro,
  author = {Hidayet Aksu},
  title  = {Maestro Order: A Model-Agnostic Orchestration Harness},
  year   = {2026},
  note   = {arXiv:2606.23983}
}

@incollection{vonneumann1956,
  author    = {John von Neumann},
  title     = {Probabilistic Logics and the Synthesis of Reliable Organisms from Unreliable Components},
  booktitle = {Automata Studies},
  series    = {Annals of Mathematics Studies},
  number    = {34},
  editor    = {Claude E. Shannon and John McCarthy},
  publisher = {Princeton University Press},
  pages     = {43--98},
  year      = {1956}
}

@article{ladha1993,
  author  = {Krishna K. Ladha},
  title   = {Condorcet's Jury Theorem in Light of de Finetti's Theorem: Majority-Rule Voting with Correlated Votes},
  journal = {Social Choice and Welfare},
  volume  = {10},
  number  = {1},
  pages   = {69--85},
  year    = {1993}
}

@misc{tsui2025selfcorrection,
  author = {Ken Tsui},
  title  = {Self-Correction Bench: Uncovering and Addressing the Self-Correction Blind Spot in Large Language Models},
  year   = {2025},
  note   = {arXiv:2507.02778}
}

@misc{limits2026,
  author = {Andrew Michael Brilliant},
  title  = {Limits of Self-Correction in LLMs: An Information-Theoretic Analysis of Correlated Errors},
  year   = {2026},
  note   = {Preprints.org 202601.0892; DOI 10.20944/preprints202601.0892.v2 (v1 2026-01-13, v2 2026-02-11); also TechRxiv}
}

@misc{halder2026judge,
  author = {Indranil Halder and Cengiz Pehlevan},
  title  = {Demystifying LLM-as-a-Judge: Analytically Tractable Model for Inference-Time Scaling},
  year   = {2025},
  note   = {arXiv:2512.19905}
}

@misc{halder2026jailbreak,
  author = {Indranil Halder and Annesya Banerjee and Cengiz Pehlevan},
  title  = {Jailbreak Scaling Laws for Large Language Models: Polynomial--Exponential Crossover},
  year   = {2026},
  note   = {arXiv:2603.11331}
}

@misc{chen2024provable,
  author = {Yanxi Chen and Xuchen Pan and Yaliang Li and Bolin Ding and Jingren Zhou},
  title  = {Provable Scaling Laws for the Test-Time Compute of Large Language Models},
  year   = {2024},
  note   = {arXiv:2411.19477; NeurIPS 2025}
}

@misc{chen2024morecalls,
  author = {Lingjiao Chen and Jared Quincy Davis and Boris Hanin and Peter Bailis and Ion Stoica and Matei Zaharia and James Zou},
  title  = {Are More LLM Calls All You Need? Towards Scaling Laws of Compound Inference Systems},
  year   = {2024},
  note   = {arXiv:2403.02419; NeurIPS 2024}
}

@misc{liu2026twocalls,
  author = {Yi Liu},
  title  = {Two Calls, Two Moments, and the Vote-Accuracy Curve of Repeated LLM Inference},
  year   = {2026},
  note   = {arXiv:2605.03379}
}

@misc{liu2026voting,
  author = {Yi Liu},
  title  = {When Can Voting Help, Hurt, or Change Course? Exact Structure of Binary Test-Time Aggregation},
  year   = {2026},
  note   = {arXiv:2605.05592}
}

@misc{chen2026cofailure,
  author = {Josef Chen},
  title  = {When Does Combining Language Models Help? A Co-Failure Ceiling on Routing, Voting, and Mixture-of-Agents Across 67 Frontier Models},
  year   = {2026},
  note   = {arXiv:2606.27288}
}

@misc{liu2026synergy,
  author = {Bang Liu and Linglong Kong and Jian Pei},
  title  = {Phase Transition for Budgeted Multi-Agent Synergy},
  year   = {2026},
  note   = {arXiv:2601.17311}
}

@misc{kohli2026ninejudges,
  author = {Guneet Kohli},
  title  = {Nine Judges, Two Effective Votes: Correlated Errors Undermine LLM Evaluation Panels},
  year   = {2026},
  note   = {arXiv:2605.29800}
}

@misc{zhou2025variation,
  author = {Yefan Zhou and Austin Xu and Yilun Zhou and Janvijay Singh and Jiang Gui and Shafiq Joty},
  title  = {Variation in Verification: Understanding Verification Dynamics in Large Language Models},
  year   = {2025},
  note   = {arXiv:2509.17995}
}

@article{kamoi2024selfcorrection,
  author  = {Ryo Kamoi and Yusen Zhang and Nan Zhang and Jiawei Han and Rui Zhang},
  title   = {When Can {LLMs} Actually Correct Their Own Mistakes? A Critical Survey of Self-Correction of {LLMs}},
  journal = {Transactions of the Association for Computational Linguistics},
  volume  = {12},
  pages   = {1417--1440},
  year    = {2024}
}

@inproceedings{huang2024selfcorrect,
  author    = {Jie Huang and Xinyun Chen and Swaroop Mishra and Huaixiu Steven Zheng and Adams Wei Yu and Xinying Song and Denny Zhou},
  title     = {Large Language Models Cannot Self-Correct Reasoning Yet},
  booktitle = {International Conference on Learning Representations (ICLR)},
  year      = {2024},
  note      = {arXiv:2310.01798}
}

@inproceedings{robbins1956,
  author    = {Herbert Robbins},
  title     = {An Empirical Bayes Approach to Statistics},
  booktitle = {Proceedings of the Third Berkeley Symposium on Mathematical Statistics and Probability},
  volume    = {1},
  pages     = {157--163},
  publisher = {University of California Press},
  year      = {1956}
}

@article{kieferwolfowitz1956,
  author  = {Jack Kiefer and Jacob Wolfowitz},
  title   = {Consistency of the Maximum Likelihood Estimator in the Presence of Infinitely Many Incidental Parameters},
  journal = {The Annals of Mathematical Statistics},
  volume  = {27},
  number  = {4},
  pages   = {887--906},
  year    = {1956}
}

@article{efron2016,
  author  = {Bradley Efron},
  title   = {Empirical Bayes Deconvolution Estimates},
  journal = {Biometrika},
  volume  = {103},
  number  = {1},
  pages   = {1--20},
  year    = {2016}
}

@article{backusgilbert1968,
  author  = {George Backus and Freeman Gilbert},
  title   = {The Resolving Power of Gross Earth Data},
  journal = {Geophysical Journal of the Royal Astronomical Society},
  volume  = {16},
  number  = {2},
  pages   = {169--205},
  year    = {1968}
}

\end{document}